\documentclass[runningheads]{llncs}
\usepackage[pdftex]{graphicx}
\usepackage{epstopdf}

\usepackage{epsfig,amssymb,color}
\usepackage{amsmath}
\usepackage{bm}
\usepackage{booktabs}
\usepackage{color}
\usepackage{multicol}
\usepackage{multirow}
\usepackage[noorphans]{quoting}

\usepackage[ruled,vlined,algo2e]{algorithm2e}
\newtheorem{defi}{Definition}
\newtheorem{prop}[defi]{Proposition}

\newcommand{\proofend}{\hfill$\Box$\vspace{2mm}}

\newcommand{\argmin}{\mathop{\mathrm{argmin\,}}}

\newcommand{\iid}{\stackrel{\mathrm{i.i.d.}}{\sim}}

\newcommand{\mathbbR}{\mathbb{R}}

\newcommand{\boldone}{{\boldsymbol{1}}}

\newcommand{\boldC}{{\boldsymbol{C}}}

\newcommand{\boldH}{{\boldsymbol{H}}}
\newcommand{\boldI}{{\boldsymbol{I}}}

\newcommand{\boldK}{{\boldsymbol{K}}}
\newcommand{\boldL}{{\boldsymbol{L}}}

\newcommand{\bolda}{{\boldsymbol{a}}}
\newcommand{\boldb}{{\boldsymbol{b}}}

\newcommand{\boldh}{{\boldsymbol{h}}}

\newcommand{\boldk}{{\boldsymbol{k}}}
\newcommand{\boldl}{{\boldsymbol{l}}}

\newcommand{\boldx}{{\boldsymbol{x}}}
\newcommand{\boldy}{{\boldsymbol{y}}}

\newcommand{\boldalpha}{{\boldsymbol{\alpha}}}

\newcommand{\boldPi}{{\boldsymbol{\mathrm{\Pi}}}}

\newcommand{\boldvarphi}{{\boldsymbol{\varphi}}}

\newcommand{\calD}{{\mathcal{D}}}

\newcommand{\calX}{{\mathcal{X}}}
\newcommand{\calY}{{\mathcal{Y}}}

\newcommand{\ptr}{p}

\usepackage[misc]{ifsym}
\usepackage[hyphens]{url}

\begin{document}
\title{LSMI-Sinkhorn: Semi-supervised Mutual Information Estimation with Optimal Transport}
\titlerunning{LSMI-Sinkhorn} \toctitle{LSMI-Sinkhorn: Semi-supervised Mutual Information Estimation with Optimal Transport}
%
\author{Yanbin Liu[\Letter]$^1$\thanks{Equal contribution, $\text{\Letter}$ Corresponding Author}, Makoto Yamada$^{2,3\star}$, Yao-Hung Hubert Tsai$^4$, Tam Le$^3$, Ruslan Salakhutdinov$^4$, \and Yi Yang$^1$}
\authorrunning{Liu~Y., Yamada~M., Tsai~YH., Le~T., Salakhutdinov~R., Yang~Y.} \tocauthor{Liu~Y., Yamada~M., Tsai~YH., Le~T., Salakhutdinov~R., Yang~Y.}
%
\institute{$^1$AAII, University of Technology Sydney\\$^2$Kyoto University, $^3$RIKEN AIP, $^4$Carnegie Mellon University\\\email{csyanbin@gmail.com}}

\maketitle \setcounter{footnote}{0}             
\begin{abstract}
Estimating mutual information is an important statistics and machine learning 
problem. To estimate the mutual information from data, a common practice is preparing a set of paired samples $\{(\boldx_i,\boldy_i)\}_{i = 1}^n$ $\iid p(\boldx,\boldy)$. However, in many situations, it is difficult to obtain a large number of data pairs. 
To address this problem, we propose the semi-supervised Squared-loss Mutual Information (SMI) estimation method using a small number of paired samples and the available unpaired ones.  
We first represent SMI through the density ratio function, where the expectation is approximated by the samples from marginals and its assignment parameters. The objective is formulated using the optimal transport problem and  quadratic programming. 
Then, we introduce the \textbf{L}east-\textbf{S}quares \textbf{M}utual \textbf{I}nformation with \textbf{Sinkhorn} (\textbf{LSMI-Sinkhorn}) algorithm for efficient optimization. 
Through experiments, we first demonstrate that the proposed method can estimate the SMI without a large number of paired samples. 
Then, we show the effectiveness of the proposed LSMI-Sinkhorn algorithm on various types of machine learning problems such as image matching and photo album summarization. Code can be found at \url{https://github.com/csyanbin/LSMI-Sinkhorn}

\end{abstract}
%
%
%

\section{Introduction}
Mutual information (MI) represents the statistical independence between two random variables \cite{book:Cover+Thomas:2006}, and it is widely used in various types of machine learning applications including feature selection \cite{BMCBio:Suzuki+etal:2009a,suzuki2009mutual}, dimensionality reduction \cite{suzuki2010sufficient}, and causal inference \cite{yamada2010dependence}. More recently, deep neural network (DNN) models have started using MI as a regularizer for obtaining better representations from data such as infoVAE \cite{zhao2017infovae} and deep infoMax \cite{hjelm2018learning}. Another application is improving the generative adversarial networks (GANs) \cite{goodfellow2014generative}. For instance, Mutual Information Neural Estimation (MINE) \cite{belghazi2018mutual} was proposed to maximize or minimize the MI in deep networks and alleviate the mode-dropping issues in GANS. In all these examples, MI estimation is the core of all these applications.

In various MI estimation approaches, the probability density ratio function is considered to be one of the most important components:
\[
r(\boldx,\boldy) = \frac{p(\boldx,\boldy)}{p(\boldx)p(\boldy)}.
\]

A straightforward method to estimate this ratio is the estimation of the probability densities (i.e., $p(\boldx,\boldy)$, $p(\boldx)$, and $p(\boldy)$), followed by calculating their ratio. However, directly estimating the probability density is difficult, thereby making this two-step approach inefficient. To address the issue, Suzuki {\em et al.} \cite{suzuki2009mutual} proposed to directly estimate the density ratio by avoiding the density estimation \cite{BMCBio:Suzuki+etal:2009a,suzuki2009mutual}. Nonetheless, the abovementioned methods requires a large number of paired data when estimating the MI.

Under practical setting, we can only obtain a small number of paired samples. For example,  it requires a massive amount of human labor to obtain one-to-one correspondences from one language to another. Thus, it prevents us to easily measure the MI across languages. Hence, a research question arises: 

\begin{quoting}
\itshape
Can we perform mutual information estimation using unpaired samples and
a small number of data pairs?
\end{quoting}

To answer the above question, in this paper, we propose a semi-supervised MI estimation approach,  particularly designed for the Squared-loss Mutual Information (SMI) (a.k.a., $\chi^2$-divergence between $p(\boldx,\boldy)$ and $p(\boldx)p(\boldy)$) \cite{BMCBio:Suzuki+etal:2009a}. We first formulate the SMI estimation as the optimal transport problem with density-ratio estimation. 
Then, we propose the \textbf{L}east-\textbf{S}quares \textbf{M}utual \textbf{I}nformation with \textbf{Sinkhorn} (\textbf{LSMI-Sinkhorn}) algorithm to optimize the problem. The algorithm has the computational complexity of $O(n_x n_y)$; hence, it is computationally efficient. 
 Through experiments, we first demonstrate that the proposed method can estimate the SMI without a large number of paired samples. 
 Then, we visualize the optimal transport matrix, which is an approximation of the joint density $p(\boldx, \boldy)$, for a better understanding of the proposed algorithm. 
 Finally, for image matching and photo album summarization, we show the effectiveness of the proposed method.

The contributions of this paper can be summarized as follows:
\begin{itemize}
\item We proposed the semi-supervised Squared-loss Mutual Information (SMI) estimation approach that 
does not require a large number of paired samples.

\item We formulate mutual information estimation as a joint density-ratio fitting and optimal transport problem, and propose an efficient \textbf{LSMI-Sinkhorn} algorithm to optimize it with a monotonical decreasing guarantee. 

\item We experimentally demonstrate the effectiveness of the proposed LSMI-Sinkhorn for MI estimation, and further show its broader applications to the image matching and photo album summarization problems. 

\end{itemize}

\section{Problem Formulation}
In this section, we formulate the problem of Squared-loss Mutual Information (SMI) estimation using a small number of paired samples and a large number of unpaired samples.

Formally, let $\calX \subset \mathbbR^{d_x}$ be the domain of random variable $\boldx$ and $\calY \subset \mathbbR^{d_y}$ be the domain of another random variable $\boldy$.  Suppose we are given $n$ independent and identically distributed (i.i.d.) \emph{paired} samples:
\[
\{(\boldx_i,\boldy_i)\}_{i = 1}^{n},
\]
where 
the number of paired samples $n$ is small. 
Apart from the paired samples, we also have access to
 $n_x$ and $n_y$ i.i.d. samples from the marginal distributions:
\[
\{\boldx_i\}_{i = n+1}^{n+n_x} \iid \ptr(\boldx)~\text{and}~\{\boldy_j\}_{j = n + 1}^{n +n_y} \iid \ptr(\boldy),
\]
where the number of unpaired samples $n_x$ and $n_y$ is much larger than that of paired samples $n$ (e.g., $n = 10$ and $n_x = n_y = 1000$).
We also denote $\boldx'_i = \boldx_{i -n}, i \in \{n+1,n+2, \ldots, n + n_x\}$ and $\boldy'_j = \boldy_{j -n}, j \in \{n+1,n+2, \ldots, n + n_y\}$, respectively. 
Note that the input dimensions $d_x$, $d_y$  and the number of samples $n_x$, $n_y$ may be different.  

This paper aims to estimate the SMI \cite{BMCBio:Suzuki+etal:2009a} (a.k.a.,  $\chi^2$-divergence between $p(\boldx,\boldy)$ and $p(\boldx)p(\boldy)$) 
from $\{(\boldx_i,\boldy_i)\}_{i = 1}^{n}$ with the help of the extra unpaired samples $\{\boldx_i\}_{i = n+1}^{n+n_x}$ and $\{\boldy_j\}_{j = n+1}^{n+n_y}$.
Specifically, the SMI between random variables $X$ and $Y$ is defined as
\begin{align}
\label{eq:smi}
\textnormal{SMI}(X,Y) \!&=\! \frac{1}{2}\!\iint \! \left(r(\boldx, \boldy) \!-\! 1 \right)^2\!\! \ptr(\boldx) \ptr(\boldy)\textnormal{d}\boldx \textnormal{d}\boldy, 
\end{align}
where 
$r(\boldx, \boldy) = \frac{\ptr(\boldx,\boldy)}{p(\boldx)p(\boldy)}$
is the density-ratio function. 
\emph{SMI takes 0 if and only if $X$ and $Y$ are independent (i.e., $p(\boldx,\boldy) = p(\boldx)p(\boldy)$), and takes a positive value if they are not independent.}

Naturally, if we know the estimation of the density-ratio function, then we can approximate the SMI in Eq.~\ref{eq:smi} as 
\begin{align*}
\widehat{\textnormal{SMI}}(X,Y)  \!&=\! \frac{1}{2(n\!+\!n_x) (n\!+\!n_y)}\!\sum_{i =1}^{n+n_x}\!\sum_{j = 1}^{n+n_y}\!\! \left( r_{\boldalpha}(\boldx_i,\boldy_j) \!-\!1\right)^2,
\end{align*}
where $r_{\boldalpha}(\boldx,\boldy)$ is an estimation of the true density ratio function $r(\boldx,\boldy)$ parameterized by $\boldalpha$. More details are discussed in \S\ref{subsec:LSMI-Sinkhorn}. 

However, in many real applications, it is difficult or laborious to obtain sufficient paired samples for density ratio estimation, which may result in high variance and bias when computing the SMI. 
In this paper, the key idea is to align the unpaired samples under this limited number of paired samples setting, and propose an objective to incorporate both the paired samples and aligned samples for a better SMI estimation.

\section{Methodology}
In this section, we propose the SMI estimation algorithm with limited number of paired samples and large number of unpaired samples.

\subsection{Least-Squares Mutual Information with Sinkhorn Algorithm}
\label{subsec:LSMI-Sinkhorn}
We employ the following density-ratio model. 
It first samples two sets of basis vectors $\{\widetilde{\boldx}_i\}_{i = 1}^b$ and $\{\widetilde{\boldy}_i\}_{i = 1}^b$ from $\{\boldx_i\}_{i = 1}^{n+n_x}$ and $\{\boldy_j\}_{j = 1}^{n+n_y}$, then computes 
\begin{align}
\label{eq:ratio-model}
    r_\boldalpha(\boldx,\boldy) &= \sum_{\ell = 1}^b \alpha_\ell K(\widetilde{\boldx}_\ell,\boldx)L (\widetilde{\boldy}_\ell,\boldy) = \boldalpha^\top \boldvarphi(\boldx,\boldy),
\end{align}
where $\boldalpha \in \mathbbR^b$, $K(\boldsymbol{\cdot}\,,\, \boldsymbol{\cdot})$ and $L(
\boldsymbol{\cdot}\,,\, \boldsymbol{\cdot})$ are kernel functions,  
$\boldvarphi(\boldx,\boldy) = \boldk(\boldx)\circ \boldl(\boldy)$ with $\boldk(\boldx) = [K(\widetilde{\boldx}_1,\boldx), \ldots, K(\widetilde{\boldx}_b,\boldx)]^\top \in \mathbbR^{b}$, $\boldl(\boldy) = [L(\widetilde{\boldy}_1,\boldy), \ldots, L(\widetilde{\boldy}_b,\boldy)]^\top \in \mathbbR^{b}$.

In this paper, we optimize $\boldalpha$ by minimizing the squared error loss between the true density-ratio function $r(\boldx,\boldy)$ and its parameterized model $r_\boldalpha(\boldx,\boldy)$:
\begin{align}
\label{eq:lsmi-loss}
    \text{Loss} &= \frac{1}{2}\iint \left(r_\boldalpha(\boldx,\boldy) - \frac{p(\boldx,\boldy)}{p(\boldx)p(\boldy)} 
    \right)^2 p(\boldx)p(\boldy)\text{d}\boldx \text{d}\boldy \nonumber \\
    &= \frac{1}{2}\iint r_\boldalpha(\boldx,\boldy)^2 p(\boldx)p(\boldy)\text{d}\boldx \text{d}\boldy 
    - \iint r_\boldalpha(\boldx,\boldy)p(\boldx,\boldy)\text{d}\boldx \text{d}\boldy 
    + \text{const.} 
\end{align}
For the first term of Eq.~\eqref{eq:lsmi-loss}, we can approximate it by using a large number of unpaired samples as it only involves $p(\boldx), p(\boldy)$. 
However, to approximate the second term, paired samples from the joint distribution (i.e., $p(\boldx,\boldy)$) are required. 
Since we only have a limited number of paired samples in our setting, the approximation of the second term may have high bias and variance. 

To deal with this issue, we leverage the abundant unpaired samples to help the approximation of the second term. 
Since we have no access to the true pair information for these unpaired samples, we propose a practical way to estimate their pair information. 
Specifically, we introduce a matrix $\mathrm{\boldPi}$ ($\pi_{ij} \geq 0$, $\sum_{i = 1}^{n_x}\sum_{j = 1}^{n_y} \pi_{i,j} = 1$) that can be regarded as a parameterized estimation of the joint density function $p(\boldx,\boldy)$. Then, we approximate the second term of Eq.~\eqref{eq:lsmi-loss} 
\begin{equation}
    \label{eq:second}
    \iint r_\boldalpha(\boldx,\boldy)p(\boldx,\boldy)\text{d}\boldx\text{d}\boldy 
    \approx \frac{\beta}{n}\sum\limits_{i =1}^n r_{\boldalpha}(\boldx_i,\boldy_i)+ (1-\beta) \sum\limits_{i =1}^{n_x}\sum\limits_{j =1}^{n_y}\pi_{ij} r_{\boldalpha}(\boldx'_i,\boldy'_j),
\end{equation}
where $0\leq \beta \leq 1$ is a parameter to balance the terms of paired and unpaired samples.   Ideally, if we can set $\pi_{ij} = \delta(\boldx'_i,\boldy'_j)/{n'}$ where $\delta(\boldx'_i,\boldy'_j)$ is $1$ for all paired $(\boldx'_i$, $\boldy'_j)$ and $0$ otherwise, and $n'$ is the total number of pairs, then we can recover the original empirical estimation (i.e., $\pi_{ij} = p(\boldx'_i,\boldy'_j)$ ideally). 

Now, we can substitute Eq.~\eqref{eq:ratio-model} and Eq.~\eqref{eq:second} back into the squared error loss function Eq.~\eqref{eq:lsmi-loss} to obtain the final loss function  as 
\begin{align*}
    J(\boldPi, \boldalpha) =  \frac{1}{2}\boldalpha^\top \boldH \boldalpha -\boldalpha^\top \boldh_{\boldPi,\beta},
\end{align*}
where
\begin{align*}
    \boldH \!&=\! \frac{1}{(n+n_x)(n+n_y)}\!\sum_{i = 1}^{n+n_x}\sum_{j = 1}^{n+n_y}\! \boldvarphi({\boldx}_i,{\boldy}_j)\boldvarphi({\boldx}_i,{\boldy}_j)^\top,\\
    \boldh_{\boldPi,\beta} &= \frac{\beta}{n}\sum_{i = 1}^n  \boldvarphi(\boldx_i,\boldy_i)+ (1-\beta)\sum_{i =1}^{n_x}\sum_{j =1}^{n_y}\pi_{ij}\boldvarphi(\boldx'_i,\boldy'_j).
\end{align*}
Since we want to estimate the density-ratio function by minimizing Eq. \eqref{eq:lsmi-loss}, the optimization problem is then given as
\begin{align}
\label{eq:smi-ot}
    \min_{\boldPi, \boldalpha} &\quad  J(\boldPi, \boldalpha) \!=\!  \frac{1}{2}\boldalpha^\top \boldH \boldalpha -\boldalpha^\top \boldh_{\boldPi,\beta} \!+\! \epsilon H(\boldPi) \!+\! \frac{\lambda}{2}\|\boldalpha\|_2^2 \nonumber \\
    \text{s.t.} &\quad  \boldPi \boldone_{n_y} = n_x^{-1}\boldone_{n_x} ~\text{and}~ \boldPi^\top \boldone_{n_x} = n_y^{-1}\boldone_{n_y}.
\end{align}
Here, we add several regularization terms. $H(\boldPi) = \sum_{i = 1}^{n_x}\sum_{j =1}^{n_y}\pi_{ij} (\log \pi_{ij} -1)$ is the negative entropic regularization to ensure $\boldPi$ non-negative, and $\epsilon > 0$ is the corresponding regularization parameter. 
$\|\boldalpha\|_2^2$ is the regularization on $\boldalpha$, and $\lambda \geq 0$ is the corresponding regularization parameter.

\subsection{Optimization}
The objective function $J(\boldPi,\boldalpha)$ is not jointly convex. However, if we fix one variable, it becomes a convex function for the other. Thus, we employ the alternating optimization approach (see Algorithm \ref{alg:alg}) on $\boldPi$ and $\boldalpha$, respectively. 

\begin{algorithm2e}[t]
\caption{\label{alg:alg}LSMI-Sinkhorn Algorithm.}
 Initialize $\boldPi^{(0)}$ and $\boldPi^{(1)}$ such that $\|\boldPi^{(1)} - \boldPi^{(0)}\|_F > \eta$ ($\eta$ is the stopping parameter), and $\boldalpha^{(0)}$, set the regularization parameters $\epsilon$ and $\lambda$, the number of maximum iterations $T$, and the iteration index $t = 1$. \\
\While{$t \leq T$ and $\|\boldPi^{(t)} - \boldPi^{(t-1)}\|_F > \eta$ }{
$\boldalpha^{(t+1)} = \argmin_{\boldalpha} J(\boldPi^{(t)},\boldalpha)$. \\
$\boldPi^{(t+1)} = \argmin_{\boldPi} J(\boldPi,\boldalpha^{(t+1)})$. \\
$t = t + 1$. \\
}
\textbf{return}  $\;\, \boldPi^{(t-1)}$ and $\boldalpha^{(t-1)}$.
\end{algorithm2e}

{\textbf{1) Optimizing $\boldPi$ using the Sinkhorn algorithm.} }
When fixing $\boldalpha$, the term in our objective relating to $\boldPi$ is
\begin{align*}
    \sum_{i =1}^{n_x}\sum_{j = 1}^{n_y} \pi_{ij} \boldalpha^\top \boldvarphi(\boldx'_i,\boldy'_j) &=\sum_{i =1}^{n_x}\sum_{j = 1}^{n_y} \pi_{ij} [\boldC_{\boldalpha}]_{ij},
\end{align*}
where $\boldC_{\boldalpha} = \boldK^\top \text{diag}(\boldalpha)\boldL \in \mathbbR^{n_x \times n_y}$, $\boldK = (\boldk(\boldx'_1), \boldk(\boldx'_2), \ldots, \boldk(\boldx'_{n_x})) \in \mathbbR^{b \times n_x}$, and $\boldL = (\boldl(\boldy'_1), \boldl(\boldy'_2), \ldots, \boldl(\boldy'_{n_y})) \in \mathbbR^{b \times n_y}$. This formulation can be considered as an optimal transport problem if we maximize it with respect to $\boldPi$ \cite{cuturi2013sinkhorn}. It is worth noting that the rank of $\boldC_\alpha$ is at most $b \ll \text{min}(n_x,n_y)$ with $b$ being a constant (e.g., $b = 100$), and the computational complexity of the cost matrix $\boldC_{\boldalpha}$ is $O(n_x n_y)$. The optimization problem with fixed $\boldalpha$ becomes
\begin{align}
    \nonumber
    \min_{\boldPi} &\quad  -\sum_{i =1}^{n_x}\sum_{j = 1}^{n_y} \pi_{ij}(1-\beta) [\boldC_{\boldalpha}]_{ij} +\epsilon H(\boldPi) \\
    \label{eq:convexPi}
    \text{s.t.} &\quad  \boldPi \boldone_{n_y} = n_x^{-1}\boldone_{n_x} ~\text{and}~ \boldPi^\top \boldone_{n_x} = n_y^{-1}\boldone_{n_y}\,,
\end{align}
which can be efficiently solved using the Sinkhorn algorithm \cite{cuturi2013sinkhorn,Sinkhorn-1974-Diagonal}~\footnote{In this paper, we use the log-stabilized Sinkhorn algorithm \cite{SIAM:schmitzer2019stabilized}.}. 
When $\boldalpha$ is fixed, problem~\eqref{eq:convexPi} is convex with respect to $\boldPi$.

{\textbf{2) Optimizing $\boldalpha$.} }
Next, when we fix $\boldPi$, the optimization problem becomes
\begin{align}
\label{eq:ratio-fit}
    \min_{\boldalpha} &\quad \frac{1}{2}\boldalpha^\top \boldH \boldalpha -\boldalpha^\top \boldh_{\boldPi,\beta} + \frac{\lambda}{2} \|\boldalpha\|_2^2\,.
\end{align}
Problem \eqref{eq:ratio-fit} is a quadratic programming and convex. It has an analytical solution
\begin{align}
\label{eq:analytical_solution}
    \widehat{\boldalpha} = (\boldH + \lambda \boldI_b)^{-1}\boldh_{\boldPi,\beta},
\end{align}
where $\boldI_b \in \mathbbR^{b \times b}$ is an identity matrix. Note that the $\boldH$ matrix does not depend on either $\boldPi$ or $\boldalpha$, and it is a positive definite matrix. 

{\textbf{Convergence Analysis.}}
To optimize $J(\boldPi,\boldalpha)$, we 
alternatively solve two convex optimization problems. Thus, the following property holds true. 
\begin{prop}
    \label{prop1}
    Algorithm \ref{alg:alg} will monotonically decrease the objective function $J(\boldPi,\boldalpha)$ in each iteration.
\end{prop}
\begin{proof}
	We show that  $J(\boldPi^{(t+1)},\boldalpha^{(t+1)}) \leq J(\boldPi^{(t)},\boldalpha^{(t)})$. First, because $\boldalpha^{(t+1)} = \argmin_{\boldalpha} J(\boldPi^{(t)}, \boldalpha)$ and $\boldalpha^{(t+1)}$ is the globally optimum solution, we have
    \[
    J(\boldPi^{(t)}, \boldalpha^{(t+1)}) \leq J(\boldPi^{(t)}, \boldalpha^{(t)}).
    \]
    Moreover, because $\boldPi^{(t+1)} = \argmin_{\boldPi} J(\boldPi, \boldalpha^{(t+1)})$ and $\boldPi^{(t+1)}$ is the globally optimum solution, we have
    \[
    J(\boldPi^{(t+1)}, \boldalpha^{(t+1)}) \leq J(\boldPi^{(t)}, \boldalpha^{(t+1)}).
    \]
    Therefore,
    \[
    J(\boldPi^{(t+1)}, \boldalpha^{(t+1)}) \leq J(\boldPi^{(t)}, \boldalpha^{(t)}). 
    \]  \proofend
    \end{proof}
    

{\bf Model Selection.}
Algorithm \ref{alg:alg} is dubbed as LSMI-Sinkhorn algorithm since it utilizes Sinkhorn algorithm for LSMI estimation. It includes several tuning parameters (i.e., $\lambda$ and $\beta$) and determining the model parameters is critical to obtain a good estimation of SMI. 
Accordingly, we use the cross-validation with the hold-out set to select the model parameters. 

First, the paired samples $\{(\boldx_i,\boldy_i)\}_{i = 1}^n$ are divided into two subsets $\calD_{\text{tr}}$ and $\calD_{\text{te}}$. Then, we train the density-ratio $r_{\boldalpha}(\boldx,\boldy)$ using $\calD_{\text{tr}}$ and the unpaired samples: $\{\boldx_i\}_{i = n+1}^{n + n_x}$ and $\{\boldy_j\}_{j = n+1}^{n + n_y}$. The hold-out error can be calculated by approximating Eq. \eqref{eq:lsmi-loss} using the hold-out samples $\calD_{\text{te}}$ as
\begin{align*}
    \widehat{J}_{\text{te}} = \frac{1}{2|\calD_{\text{te}}|^2} \sum_{\boldx,\boldy \in \calD_{\text{te}}}r_{\widehat{\boldalpha}}(\boldx,\boldy)^2 - \frac{1}{|\calD_{\text{te}}|} \sum_{(\boldx,\boldy) \in \calD_{\text{te}}} r_{\widehat{\boldalpha}}(\boldx,\boldy),
\end{align*}
where $|\calD|$ denotes the number of samples in the set $\calD$, $\sum_{\boldx,\boldy \in \calD_{\text{te}}}$ denotes the summation over all possible combinations of $\boldx$ and $\boldy$ in $\calD_{\text{te}}$, and $\sum_{(\boldx,\boldy) \in \calD_{\text{te}}}$ denotes the summation over all pairs of $(\boldx, \boldy)$ in $\calD_{\text{te}}$. We select the parameters that lead to the smallest $\widehat{J}_\text{te}$.

\subsection{Discussion}
{\bf Relation to Least-Squares Object Matching (LSOM).}
In this section, we show that the LSOM algorithm \cite{yamada2015cross,yamada2011cross} can be considered as a special case of the proposed framework. %
If $\boldPi$ is a permutation matrix and $n' = n_x = n_y$,
\begin{align*}
    \boldPi = \{0,1\}^{n'\times n'},~\boldPi \boldone_{n'} = \boldone_{n'},~\text{and}~ \boldPi^\top \boldone_{n'} = \boldone_{n'},
\end{align*}
where $\boldPi^\top \boldPi = \boldPi \boldPi^\top = \mathrm{\boldI}_{n'}$. 
Then, the estimation of SMI using the permutation matrix can be written as
\begin{align*}
    \widehat{\text{SMI}}(X,Y) 
    = \frac{\beta}{2n}\sum_{i =1}^n r_{\boldalpha}(\boldx_i,\boldy_i) + \frac{1}{2n'}\sum_{i =1}^{n'}(1-\beta) r_{\boldalpha}(\boldx'_i,\boldy'_{\pi_{(i)}}) - \frac{1}{2},
\end{align*}
where $\pi(i)$ is the permutation function. In order to calculate $\widehat{\text{SMI}}(X,Y)$, the optimization problem is written as
\begin{align*}
    \min_{\boldPi, \boldalpha} &\quad  \frac{1}{2}\boldalpha^\top \boldH \boldalpha  -\boldalpha^\top \boldh_{\boldPi,\beta} + \frac{\lambda}{2}\|\boldalpha\|_2^2 \nonumber \\
    \text{s.t.} &\quad \boldPi \boldone_{n'} = \boldone_{n'},~ \boldPi^\top \boldone_{n'} = \boldone_{n'},~ \boldPi \in \{0,~1\}^{n' \times n'}.
\end{align*}

To solve this problem, LSOM uses the Hungarian algorithm \cite{NRLQ:Kuhn:1955} instead of the Sinkhorn algorithm \cite{cuturi2013sinkhorn} for optimizing $\boldPi$. It is noteworthy that in the original LSOM algorithm, the permutation matrix is introduced to permute the Gram matrix (i.e., $\boldPi \boldL \boldPi^\top$) and $\boldPi$ is also included within the $\boldH$ computation. However, in our formulation, the permutation matrix depends only on $\boldh_{\boldPi, \beta}$. This difference enables us to show a monotonic decrease for the loss function of the proposed algorithm.  

Since LSOM aims to seek the alignment, it is more suitable to find the exact matching among samples when the exact matching exists. 
In contrast, the proposed LSMI-Sinkhorn is reliable even when there is no exact matching. 
Moreover, LSOM assumes the same number of samples (i.e., $n_x = n_y$), while our LSMI-Sinkhorn does not have this constraint. 
For computational complexity, the Hungarian algorithm requires $O({n'}^3)$ while the Sinkhorn requires $O({n'}^2)$.

{\bf Computational Complexity.} 
First, the computational complexity of estimating $\boldPi$ is based on the computation of the cost matrix $\boldC_\boldalpha$ and the Sinkhorn iterations. The computational complexity of $\boldC_{\boldalpha}$ is $O(n_x n_y)$ and that of Sinkhorn algorithm is $O(n_xn_y)$. Therefore, the computational complexity of the Sinkhorn iteration is $O(n_x n_y)$. 
Second, for the $\boldalpha$ computation, the complexity to compute $\boldH$ is $O((n+n_x)^2 + (n + n_y)^2)$ and that for $\boldh_{\boldPi, \beta}$ is $O(n_xn_y)$. In addition, estimating $\boldalpha$ has the complexity $O(b^3)$, which is negligible with a small constant $b$. 
To conclude, the total computational complexity of the initialization needs $O((n+n_x)^2 + (n + n_y)^2)$ and the iterations requires $O(n_x n_y)$. 
In particular, for small $n$ and large $n_x = n_y$, the computational complexity is $O(n_x^2)$.

As a comparison, for another related algorithm,  Gromove-Wasserstein~\cite{memoli2011gromov,peyre2016gromov}, the time complexity of computing the objective function is $O(n_x^4)$ for general cases and $O(n_x^3)$ for some specific losses (e.g. $L_2$ loss, Kullback-Leibler loss) \cite{peyre2016gromov}. 

\section{Related Work}
In this paper, we focus on the mutual information estimation problem. 
Moreover, the proposed LSMI-Sinkhorn algorithm is related to Gromov-Wasserstein \cite{peyre2016gromov,memoli2011gromov} and kernelized sorting~\cite{PAMI:Quadrianto+etal:2010,AAAI:djuric2012CKS}.

\noindent {\bf Mutual information estimation.} To estimate the MI, a straightforward approach is to estimate the probability density $p(\boldx,\boldy)$ from the paired samples $\{(\boldx_i,\boldy_i)\}_{i = 1}^n$, $p(\boldx)$ from $\{\boldx_i\}_{i = 1}^n$, and $p(\boldy)$ from $\{\boldy_i\}_{i=1}^n$, respectively. 

Because the estimation of the probability density is itself a difficult problem, this straightforward approach does not work well. 
To handle this, a density-ratio based approach can be promising \cite{suzuki2009mutual,BMCBio:Suzuki+etal:2009a}. 
More recently, deep learning based mutual information estimation algorithms have been proposed \cite{belghazi2018mutual,ozair2019wasserstein}. However, these approaches still require a large number of paired samples to estimate the MI. 
Thus, in real world situations when we only have a limited number of paired samples, existing approaches are not effective to obtain a reliable estimation. 
 

\noindent {\bf Gromov-Wasserstein and Kernelized Sorting.} 
Given two set of vectors in different spaces, the Gromov-Wasserstein distance \cite{memoli2011gromov} can be used to find the optimal alignment between them. 
This method considers the pairwise distance between samples in the same set to build the distance matrix, then it finds a matching by minimizing the difference between the pairwise distance matrices:
\begin{align*}
    \min_{\boldPi} &\quad \sum_{i = 1}^{n_x}\sum_{j = 1}^{n_y}\sum_{i' = 1}^{n_x}\sum_{i' = 1}^{n_y}\pi_{ij}\pi_{i'j'}(D(\boldx_i,\boldx_{i'}) - D(\boldy_j,\boldy_{j'}))^2,\\
    \text{s.t.}&\quad \boldPi \boldone_{n_y} = \bolda, \boldPi^\top \boldone_{n_x} = \boldb,\pi_{ij} \geq 0, 
\end{align*}
where $\bolda \in \Sigma_{n_x}$, $\boldb \in \Sigma_{n_y}$, and $\Sigma_n = \{p\in \mathbbR_n^+;\sum_i p_i=1 \}$ is the probability simplex. 

Computing Gromov-Wasserstein distance requires solving the quadratic assignment problem (QAP), and it is generally NP-hard for arbitrary inputs \cite{peyre2016gromov,peyre2019computational}. In this work, we estimate the SMI by simultaneously solving the alignment and fitting the distribution ratio by efficiently leveraging the Sinkhorn algorithm and properties of the squared-loss. 
Recently, semi-supervised Gromov-Wasserstein-based Optimal transport has been proposed and applied to the heterogeneous domain adaptation problems \cite{yan2018semi}. 
However, their method cannot be directly used to measure the independence between two sets of random variables. 
In contrast, we can achieve this by the estimation of the density-ratio function. 

Kernelized sorting methods~\cite{PAMI:Quadrianto+etal:2010,AAAI:djuric2012CKS} are highly related to Gromov-Wasserstein. Specifically, the kernelized sorting determines a set of paired samples by maximizing the Hilbert-Schmidt independence criterion (HSIC) between samples. 
Similar to LSOM~\cite{yamada2011cross}, the kernelized sorting also has the assumption of the same number of samples (i.e., $\{\boldx'_i\}_{i = 1}^{n'}$ and $\{\boldy'_i\}_{j = 1}^{n'}$). This assumption prohibits both LSOM and kernelized sorting from being applied to a broader range of applications, such as photo album summarization in Section~\ref{sec:album}. 
To the contrary, since the proposed LSMI-Sinkhorn does not rely on this assumption, it can be applied to more general scenarios when $n_x \neq n_y$.

\section{Experiments}
In this section, we first estimate the SMI on both the synthetic data and benchmark datasets. Then, we apply our algorithm to real world applications, i.e., deep image matching and photo album summarization. 

\subsection{Setup}
For the density-ratio model, we utilize the Gaussian kernels:
$${ K(\boldx,\boldx') \!=\! \exp\left(\!\!-\!\frac{\|\boldx \!-\! \boldx'\|_2^2}{2\sigma_x^2}\!\right), 
  L(\boldy,\boldy') \!=\! \exp\left(\!\!-\!\frac{\|\boldy \!-\! \boldy'\|_2^2}{2\sigma_y^2}\!\right), }$$
where $\sigma_x$ and $\sigma_y$ denote the widths of the kernel that are set using the median heuristic \cite{sriperumbudur2009kernel} as
$ \sigma_x = 2^{-1/2} \text{median}(\{\|\boldx_i - \boldx_j\|_2\}_{i,j=1}^{n_x}),
  \sigma_y = 2^{-1/2} \text{median}(\{\|\boldy_i - \boldy_j\|_2\}_{i,j=1}^{n_y}).$ 
We set the number of basis $b = 200$, $\epsilon = 0.3$, the maximum number of iterations $T=20$, and the stopping parameter $\eta = 10^{-9}$. $\beta$ and $\lambda$ are chosen by cross-validation.

\subsection{Convergence and Runtime}

We first demonstrate the convergence of the loss function and the estimated SMI value. Here, we generate synthetic data from $\boldy=0.5\boldx+\mathcal{N}(0,0.01)$ and randomly choose $n=50$ paired samples and $n_x=n_y=500$ unpaired samples. The convergence curve is shown in Figure \ref{fig:convergence}. The loss value and SMI value converge quickly ($<$5 iterations), which is consistent with Proposition \ref{prop1}.

\begin{figure}[t]
    \centering
    \includegraphics[width=0.75\textwidth]{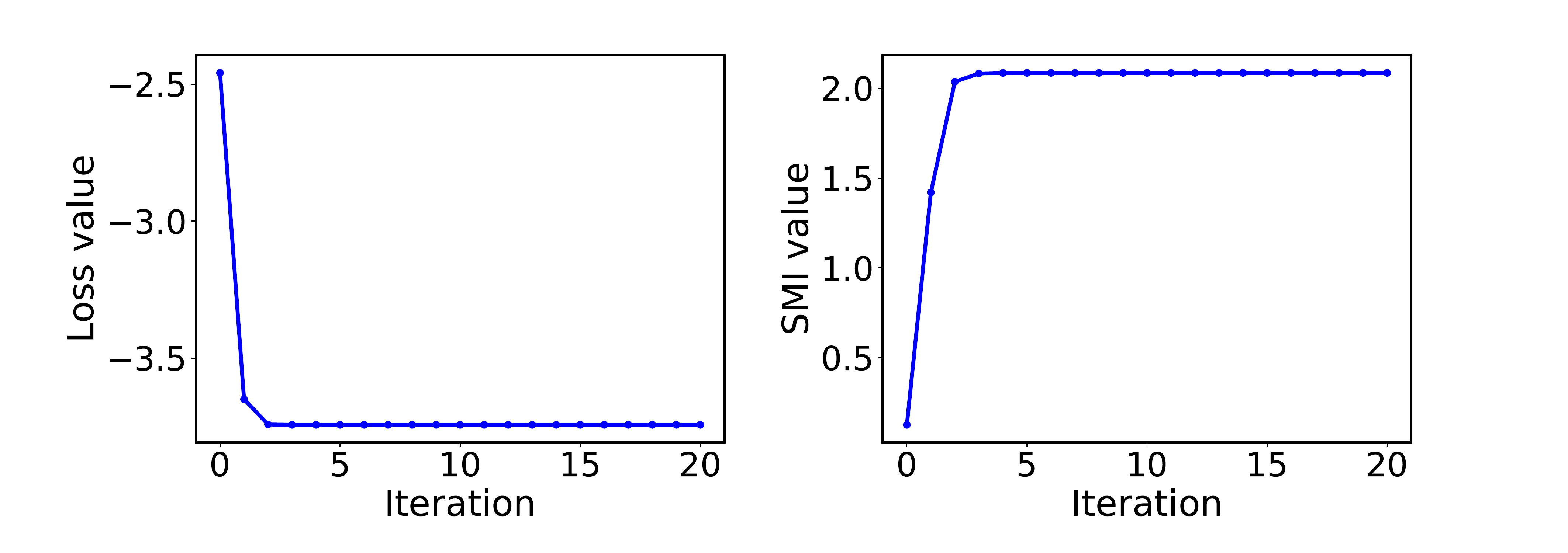}
    \caption{Convergence curves of the loss and SMI values.} 
    \label{fig:convergence}
\end{figure}

\begin{figure}[t]
    \centering
    \includegraphics[width=0.6\textwidth]{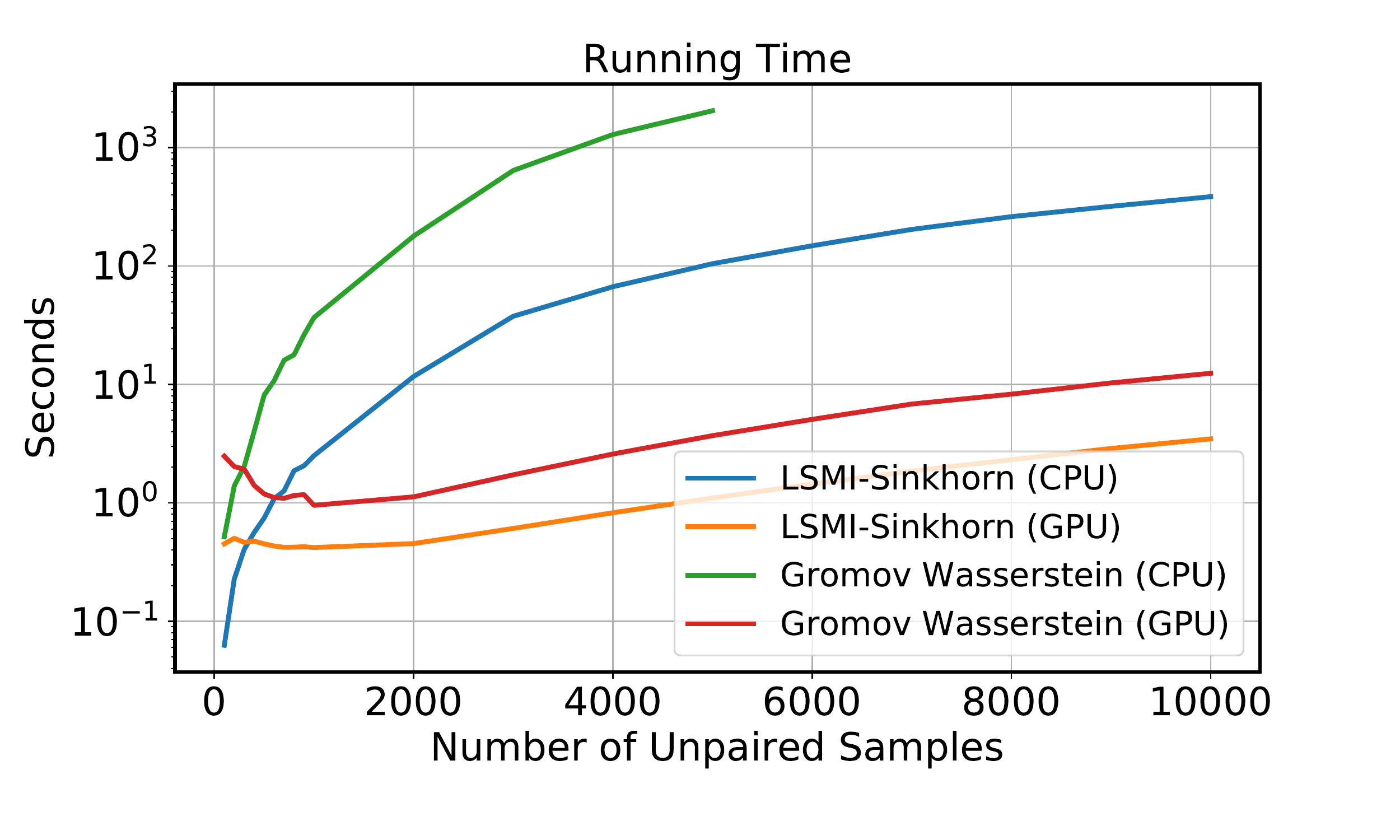}
    \caption{\label{fig:time}Runtime comparison of LSMI-Sinkhorn and Gromov-Wasserstein. A base-10 log scale is used for the Y axis.}
\end{figure}

Then, we perform a comparison between the runtimes of the proposed LSMI-Sinkhorn and Gromov-Wasserstein for CPU and GPU implementations. The data are sampled from two 2D random measures, where $n_x=n_y \in \{100,200,\dots,$ $9000,10000\}$ is the number of unpaired data and $n=100$ is the number of paired data (only for LSMI-Sinkhorn). For Gromov-Wasserstein, we use the CPU implementation from Python Optimal Transport toolbox \cite{flamary2017pot} and the Pytorch GPU implementation from \cite{ICML19:bunnePytorchGW}. We use the squared loss function and set the entropic regularization $\epsilon$ to 0.005 according to the original code. For LSMI-Sinkhorn, we implement the CPU and GPU versions using numpy and Pytorch, respectively. For fair comparison, we use the log-stabilized Sinkhorn algorithm and the same early stopping criteria and the same maximum iterations as in Gromov-Wasserstein. As shown in Figure \ref{fig:time}, in comparison to the Gromov-Wasserstein, LSMI-Sinkhorn is more than one order of magnitude faster for the CPU version and several times faster for the GPU version. This is consistent with our computational complexity analysis. Moreover, the GPU version of our algorithm costs only 3.47s to compute $10,000$ unpaired samples, indicating that it is suitable for large-scale applications.

\begin{figure}[!htb]
    \centering
    \includegraphics[width=0.8\textwidth]{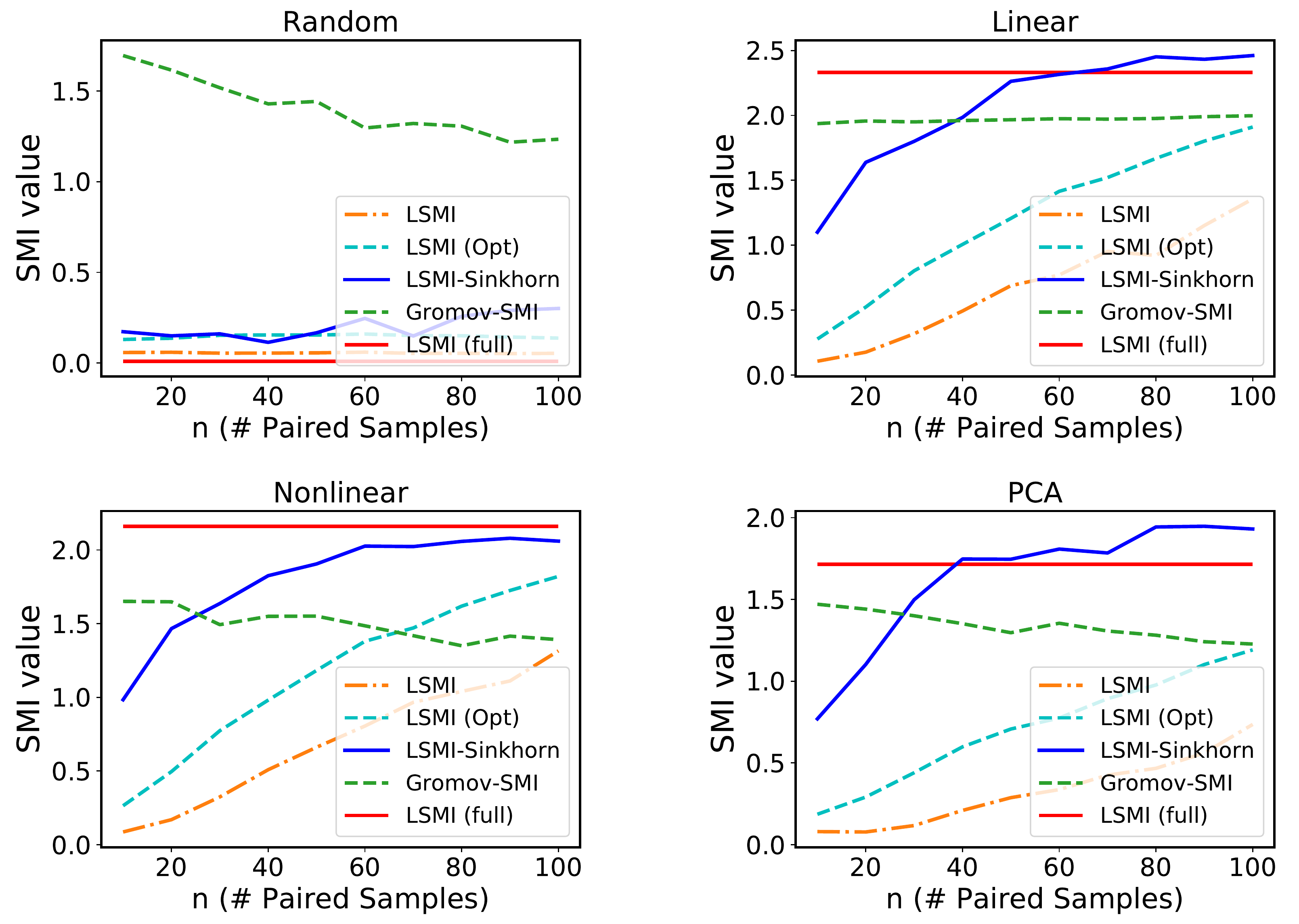}
    \caption{SMI estimation on synthetic data (\small{$n_x=n_y=500$)}.}
    \label{fig:synthetic}
\end{figure}

\begin{figure}[!htb]
    \centering
    \includegraphics[width=0.65\textwidth]{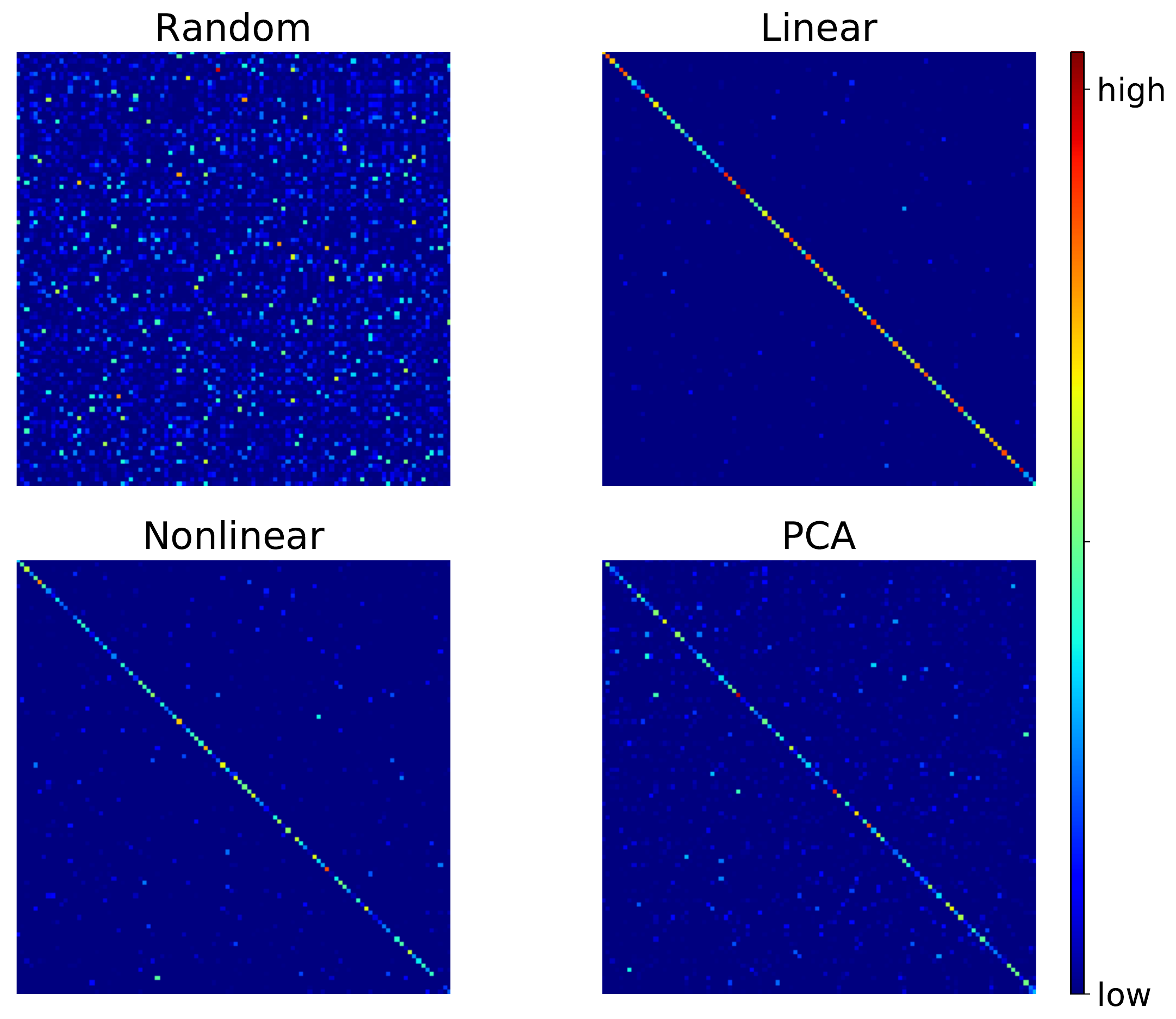}
    \caption{Visualization of the matrix $\boldPi$.}
    \label{fig:heatmap}
\end{figure}

\subsection{SMI Estimation}
For SMI estimation, we set up four baselines:
\begin{itemize}
\item \textbf{LSMI (full)}: $10,000$ paired samples are used for cross-validation and SMI estimation. It is considered as the ground truth value.
\item \textbf{LSMI}: Only $n$ (usually small) paired samples are used for cross-validation and SMI estimation. 
\item \textbf{LSMI (opt)}: $n$ paired samples are used for SMI estimation. However, we use the optimal parameters from LSMI (full) here. This can be seen as the upper bound of SMI estimation with limited number of paired data because the optimal parameters are usually unavailable. 
\item \textbf{Gromov-SMI}: The Gromov-Wasserstein distance is applied on unpaired samples to find potential matching ($\hat{n}=\min(n_x,n_y)$). Then, the $\hat{n}$ matched pairs and existing $n$ paired samples are combined to perform cross-validation and SMI estimation.
\end{itemize}

\begin{figure}[t]
    \centering
    \includegraphics[width=0.8\textwidth]{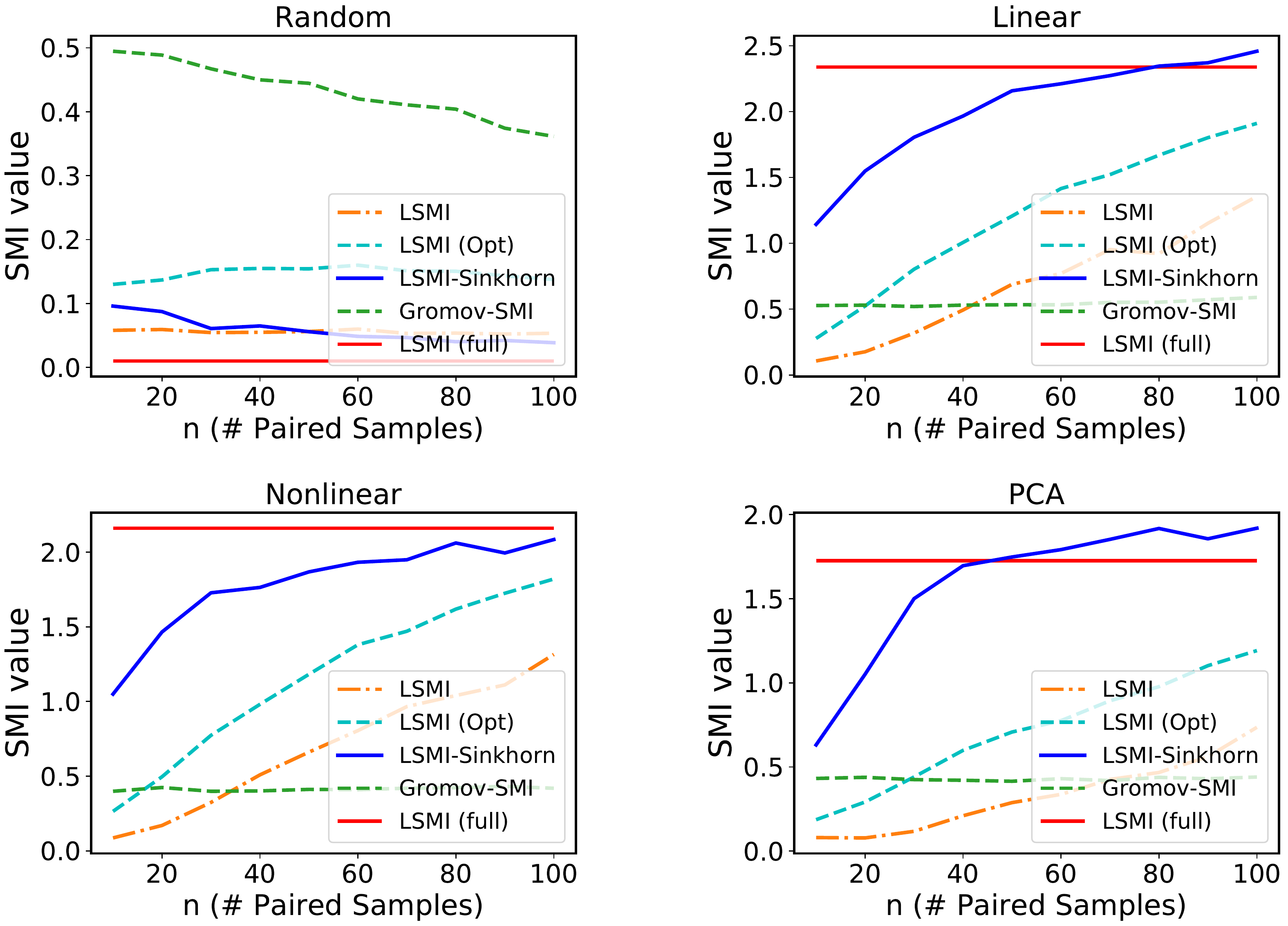}
    \caption{SMI estimation on synthetic data (\small{$n_x=1000, n_y=500$)}.}
    \label{fig:synthetic2}
\end{figure}

\noindent {\bf Synthetic Data.}
In this experiment, we manually generate four types of paired samples: random normal, $\boldy=0.5\boldx+\mathcal{N}(0,0.01)$ (Linear), $\boldy=\sin(\boldx)$ (Nonlinear), and $\boldy=\text{PCA}(\boldx)$. We change the number of paired samples $n\in\{10,20,\dots,100\}$ while \textcolor{black}{fixing $n_x=500$ and $n_y=500$} for Gromov-SMI and the proposed LSMI-Sinkhorn, respectively. The model parameters $\lambda$ and $\beta$ are selected by cross-validation using the paired examples with $\lambda \in \{0.1,0.01,0.001,0.0001\}$ and $\beta \in \{0.2,0.4,0.6,0.8,1.0\}$. The results are shown in Figure \ref{fig:synthetic}. In the random case, the data are nearly independent and our algorithm achieves a small SMI value. In other cases, LSMI-Sinkhorn yields a better estimation of the SMI value and it lies near the ground truth when $n$ increases. In contrast, Gromov-SMI has a small estimation value, which may be due to the incorrect potential matching. We further show the heatmaps of the matrix $\boldPi$ in Figure \ref{fig:heatmap}. For the random case, $\boldPi$ distributes uniformly as expected. For all other cases, $\boldPi$ concentrate on the diagonal, indicating good estimation for the unpaired samples.

To show the flexibility of the proposed LSMI-Sinkhorn algorithm, we set $n_x=1000, n_y=500$ and fix all other settings. The results are shown in Figure~\ref{fig:synthetic2}. Similarly, LSMI-Sinkhorn achieves the best performance among all methods. We also notice that Gromov-SMI achieves even worse estimation than $n_x=n_y$ case, which means it is not as stable as our algorithm to handle sophisticated situations ($n_x \neq n_y$).


\begin{figure}[t]
    \centering
    \includegraphics[width=0.8\textwidth]{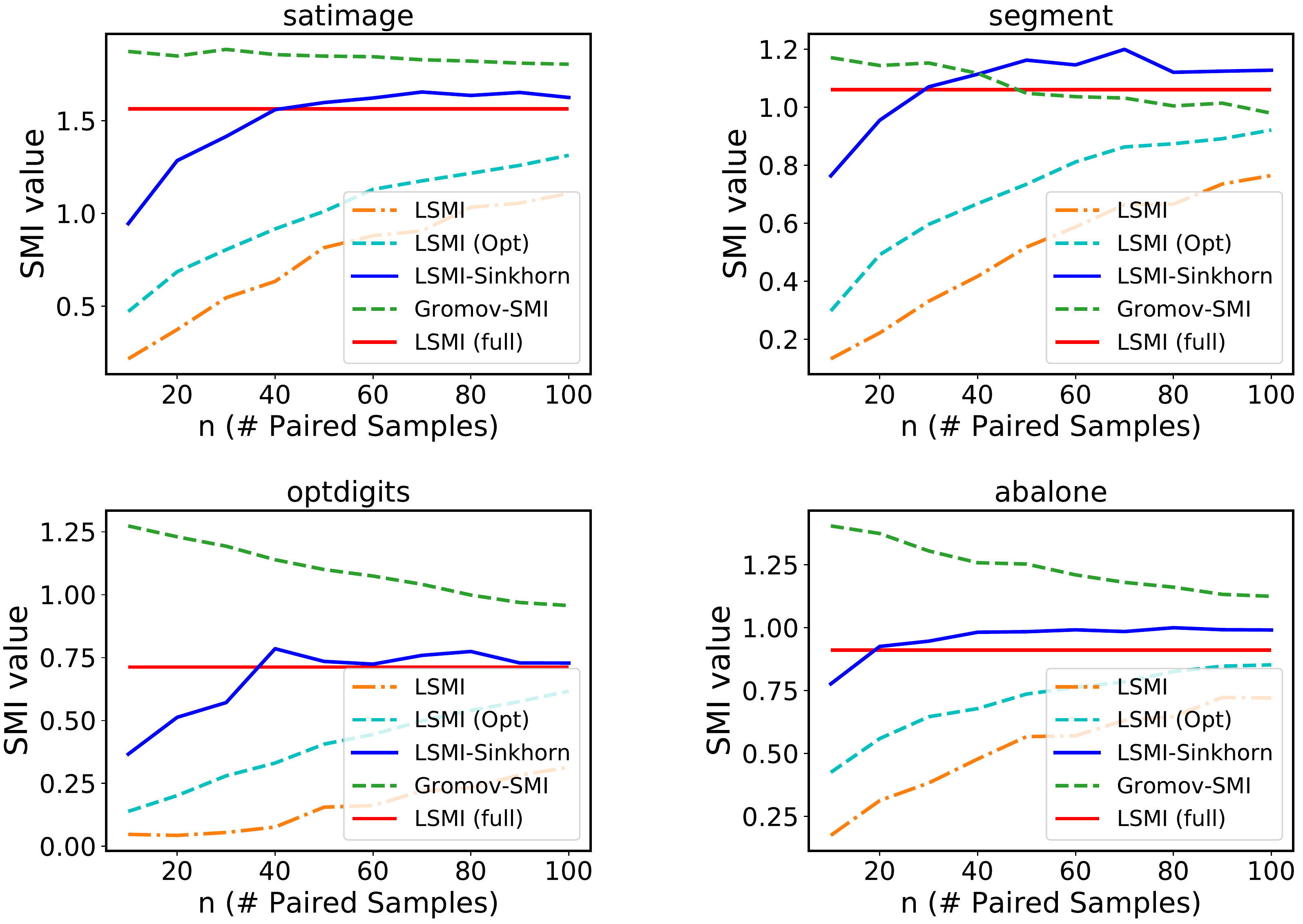}
    \caption{SMI estimation on UCI datasets.}
    \label{fig:uci}
\end{figure}

\noindent{\bf UCI Datasets.}
We selected four benchmark datasets from the UCI machine learning repository. For each dataset, we split the features into two sets as paired samples. To ensure high dependence between these two subsets of features, we utilized the same splitting strategy as \cite{PAMI:Quadrianto+etal:2010} according to the correlation matrix. \textcolor{black}{The experimental setting is the same as the synthetic data experiment}. We show the SMI estimation results in Figure \ref{fig:uci}. Similarly, LSMI-Sinkhorn obtains better estimation values in all four datasets. Gromov-SMI tends to overestimate the value by a large margin, while other baselines underestimate the value. 


\begin{figure}[t]
    \centering
    \includegraphics[width=0.8\textwidth]{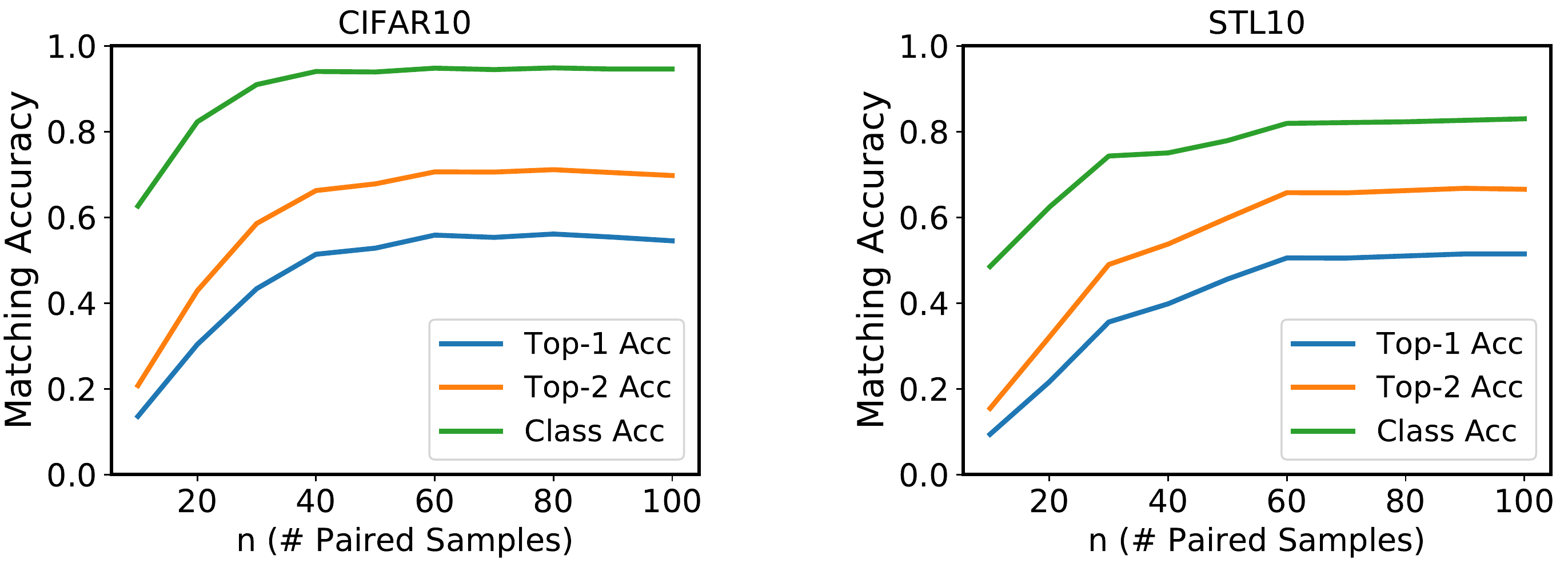}
    \caption{Deep image matching.}
    \label{fig:matching}
\end{figure}

\begin{figure}[!htb]
	\centering
	\includegraphics[height=1.5in]{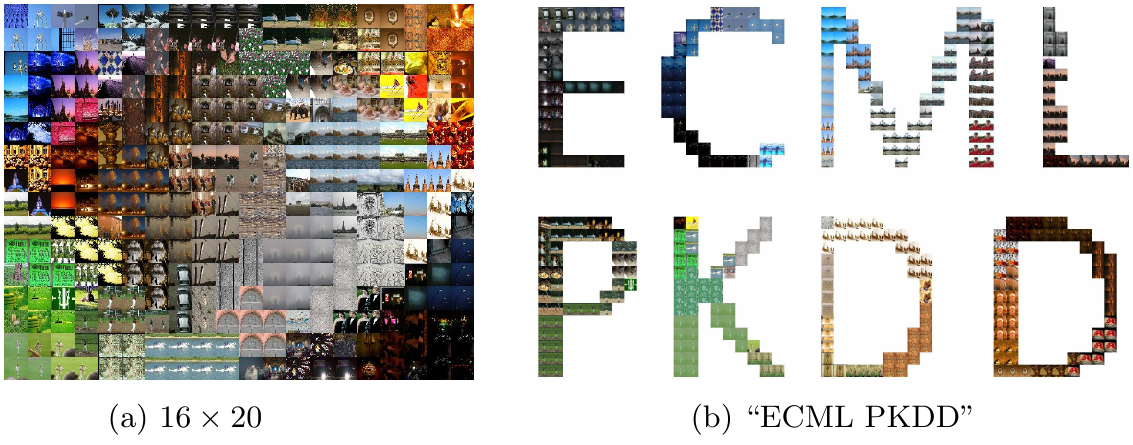}
	\caption{\label{fig:layout}Photo album summarization on Flickr dataset. In (a), we fixed the corners with \emph{blue}, \emph{orange}, \emph{green}, and \emph{black} images. 
		In (b), we fixed the center of each character with a different image.}
\end{figure}

\begin{figure}[!htb]
	\centering
	\includegraphics[height=1.5in]{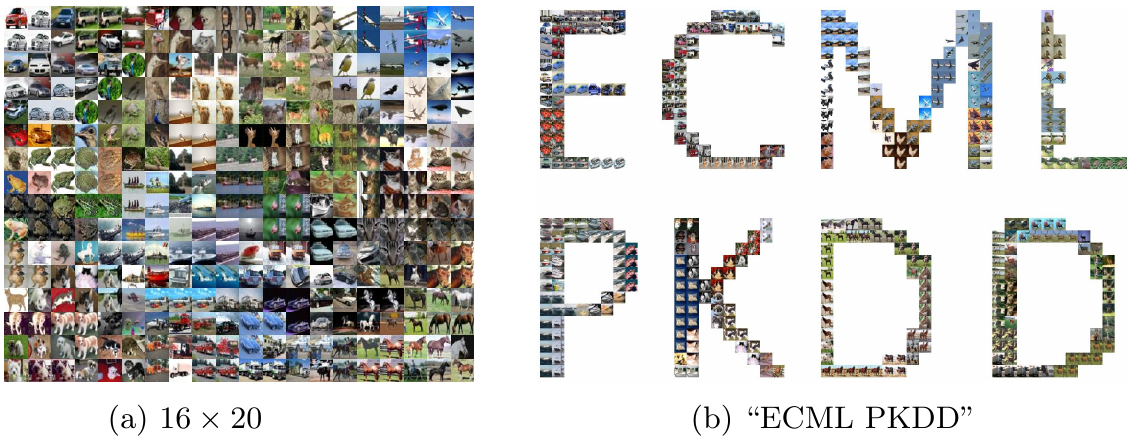}
	\caption{\label{fig:layout_cifar}Photo album summarization on CIFAR10 dataset. In (a), we fixed the corners with \emph{automobile}, \emph{airplane}, \emph{dog}, and \emph{horse} images. In (b), we fixed the center of each character with a different image.}
\end{figure}

\subsection{Deep Image Matching}
Next, we consider an image matching task with deep convolution features. We use two commonly-used image classification benchmarks: CIFAR10~
and STL10~\cite{AISTATS:coates2011STL10}. 
We extracted 64-dim features from the last layer (after pooling) of ResNet20 
pretrained on the training set of CIFAR10. The features are divided into two 32-dim parts denoted by $\{\boldx_i\}_{i=1}^N$ and $\{{\boldy_i}\}_{i=1}^N$. We shuffle the samples of $\boldy$ and attempt to match $\boldx$ and $\boldy$ with limited pair samples $(n \in \{10,20,\dots,100\})$ and unpaired samples $(n_x=n_y=500)$. Other settings are the same as the above experiments. 

To evaluate the matching performance, we used top-1 accuracy, top-2 accuracy (correct matching is achieved in the top-2 highest scores), and class accuracy (matched samples are in the same class). As shown in Figure \ref{fig:matching}, LSMI-Sinkhorn obtains high accuracy with only a few tens of supervised pairs. Additionally, the high class matching performance implies that our algorithm can be applied to further applications such as semi-supervised image classification.

\subsection{Photo Album Summarization}
\label{sec:album}

Finally, we apply the proposed LSMI-Sinkhorn to the photo album summarization problem, where images are matched to a predefined structure according to the Cartesian coordinate system.

\noindent{\bf Color Feature.} We first used 320 images from Flickr \cite{PAMI:Quadrianto+etal:2010} and extracted the RGB pixels as color feature. Figure~\ref{fig:layout} depicts the semi-supervised summarization to the $16\times 20$ grid with the corners of the grid fixed to \emph{blue}, \emph{orange}, \emph{green}, and \emph{black} images. Similarly, we show the summarization results on an ``ECML PKDD" grid with the center of each character fixed. It can be seen that these layouts show good color topology according to the fixed color images.

\noindent{\bf Semantic Feature.} We then used CIFAR10 with the ResNet20 feature to illustrate the semantic album summarization. Figure~\ref{fig:layout_cifar} shows the layout of 1000 images into the same $16\times 20$, and ``ECML PKDD" grids. In Figure~\ref{fig:layout_cifar}a, we fixed corners of the grid to \emph{automobile}, \emph{airplane}, \emph{dog}, and \emph{horse} images. In Figure \ref{fig:layout_cifar}b, we fixed the eight character centers. It can be seen that objects are aligned together by their semantics rather than colors according to the fixed images. 

Compared with previous summarization algorithms, LSMI-Sinkhorn has two advantages. (1) The semi-supervised property enables interactive album summarization, while kernelized sorting \cite{PAMI:Quadrianto+etal:2010,AAAI:djuric2012CKS} and object matching \cite{yamada2015cross} can not. (2) We obtained a solution for general rectangular matching (both $n_x=n_y$ and $n_x \neq n_y$), e.g., 320 images to a $16\times 20$ grid, 1000 images to a $16\times 20$ grid, while most previous methods \cite{PAMI:Quadrianto+etal:2010,yamada2015cross} relied on the Hungarian algorithm \cite{NRLQ:Kuhn:1955} to obtain square matching ($n_x = n_y$) only. 

\section{Conclusion}
In this paper, we proposed the Least-Square Mutual Information with Sinkhorn (LSMI-Sinkhorn) algorithm to estimate the SMI from a limited number of paired samples. To the best of our knowledge, this is the first semi-supervised SMI estimation algorithm. Experiments on synthetic and real data show that the proposed algorithm can successfully estimate SMI with a small number of paired samples. Moreover, we demonstrated that the proposed algorithm can be used for image matching and photo album summarization.

\section*{Acknowledgements}
Yanbin Liu and Yi Yang are supported by ARC DP200100938. Makoto Yamada was supported by MEXT KAKENHI 20H04243 and partly supported by MEXT KAKENHI 21H04874. 
Tam Le acknowledges the support of JSPS KAKENHI Grant number 20K19873. Yao-Hung Hubert Tsai and Ruslan Salakhutdinov were supported in part by the NSF IIS1763562, IARPA
D17PC00340, ONR Grant N000141812861, and Facebook PhD Fellowship.

\bibliographystyle{splncs04}

\end{document}